\documentclass{article}

\usepackage{amsmath,amssymb,amsthm,fullpage}
\theoremstyle{plain}
\newtheorem{lemma}{Lemma}
\newtheorem{theorem}{Theorem}

\usepackage[english]{babel}
\usepackage[a4paper,top=2cm,bottom=2cm,left=3cm,right=3cm,marginparwidth=1.75cm]{geometry}

\usepackage{float}
\usepackage{siunitx}
\usepackage{graphicx}
\PassOptionsToPackage{hyphens}{url}\usepackage{hyperref}
\usepackage{cleveref}
\usepackage[utf8]{inputenc}
\usepackage[right]{lineno}
\usepackage{csquotes}
\usepackage{booktabs,longtable,adjustbox,array}
\usepackage{titlesec}
\usepackage{authblk}
\usepackage{xcolor}

\titleformat{\subsection}
  {\mdseries\itshape\large}{\thesubsection}{1em}{}

\crefformat{figure}{#2Figure~#1#3}
\crefformat{table}{#2Table~#1#3}
\crefformat{section}{#2Section~#1#3}

\usepackage[round,authoryear]{natbib}   
\author[1]{Max Hartman}
\author[1,2]{Lav R.~Varshney}
\affil[1]{University of Illinois, Electrical \& Computer Engineering Department}
\affil[2]{University of Illinois, Coordinated Science Laboratory}

\title{SparseJEPA: Sparse Representation Learning of Joint Embedding Predictive Architectures}
\date{}

\begin{document}
\maketitle

\begin{abstract}
Joint Embedding Predictive Architectures (JEPA) have emerged as a powerful framework for learning general-purpose representations. However, these models often lack interpretability and suffer from inefficiencies due to dense embedding representations. We propose SparseJEPA, an extension that integrates sparse representation learning into the JEPA framework to enhance the quality of learned representations. SparseJEPA employs a penalty method that encourages latent space variables to be shared among data features with strong semantic relationships, while maintaining predictive performance. We demonstrate the effectiveness of SparseJEPA by training on the CIFAR-100 dataset and pre-training a lightweight Vision Transformer. The improved embeddings are utilized in linear-probe transfer learning for both image classification and low-level tasks, showcasing the architecture's versatility across different transfer tasks. Furthermore, we provide a theoretical proof that demonstrates that the grouping mechanism enhances representation quality. This was done by displaying that grouping reduces Multiinformation among latent-variables, including proofing the Data Processing Inequality for Multiinformation. Our results indicate that incorporating sparsity not only refines the latent space but also facilitates the learning of more meaningful and interpretable representations. In further work, hope to further extend this method by finding new ways to leverage the grouping mechanism through object-centric representation learning.
\end{abstract}

\section{Introduction}
Joint Embedding Predictive Architectures (JEPA) seek to learn comprehensive representations by predicting masked target blocks provided a single context block \citep{assran2023selfsupervisedlearningimagesjointembedding}. These architectures learn such representations by making predictions in the latent-space of the data they seek to reconstruct. In the case of images, for example, the context block and target blocks are embedded via a Vision Transformer. These models have seen great results in Computer Vision downstream tasks such as Image Classification, Object Counting, and Depth Perception, compared to similar Self-Supervised Architectures like Masked Autoencoders (MAE) \citep{he2021maskedautoencodersscalablevision}, data2vec \citep{baevski2022data2vecgeneralframeworkselfsupervised}, and Context Auto Encoders (CAE) \citep{chen2023contextautoencoderselfsupervisedrepresentation}  \par
Despite the success of these models to learn abstract representations, they rely on dense embeddings. This high density also poses a challenge from an interpretability standpoint. For future work on learning representations for World Knowledge, understanding how representations group and store knowledge will be critical for further downstream tasks across Robotics and Computer Vision. Despite their effectiveness, the dense embedding representations in JEPA act as "black boxes," making it difficult to understand the factors driving their predictions. Sparse representation learning provides an opportunity to make these embeddings more interpretable, yielding insights into the underlying structure of the data.\par
Sparse representation learning offers a transformative approach to addressing inefficiencies and interpretability challenges in JEPA. By introducing sparsity into the latent space, models can focus on the most relevant features, enabling them to learn meaningful representations that uncover intrinsic structures within the data.\par
Sparse representations go beyond mere computational benefits by naturally aligning the latent space with meaningful patterns in the data. In the case of SparseJEPA, our results demonstrate that sparsity helps the model learn a latent space that captures specific groupings of shared strong semantic relationships across input images. This grouping reflects the underlying structure of the data, allowing embeddings to emphasize regions or features that share semantic or informational similarities. For example, in visual datasets, sparsity can lead to embeddings that cluster objects or textures based on shared attributes, revealing deeper insights into the data distribution. This alignment of sparse embeddings with shared semantic relationships enhances both the interpretability and utility of the model\par
In the broader field of machine learning, techniques such as $L_1$-regularization, structured sparsity, and group sparsity have been instrumental in promoting sparse solutions. Output-Interpretable Variational Autoencoders have been seen to enhance latent-space representations, while maintaining model performance \citep{pmlr-v80-ainsworth18a}. $L_1$-regularization penalizes the magnitude of latent variables, effectively shrinking less relevant features to zero, while structured sparsity encourages hierarchical or grouped patterns in embeddings. These methods have demonstrated success in producing compact and interpretable representations. However, their integration into JEPA frameworks remains relatively unexplored, representing a key opportunity to leverage these advancements in a joint embedding context.\par
This work proposes an architecture extending JEPA by including a sparsity loss penalty to enhance interpretability and efficiency in the learned latent representations. The sparsity loss penalty employed is adapted from oi-VAE: Output Interpretable VAEs for Nonlinear Group Factor Analysis. This loss encourages the latent space to reflect groupings of shared semantic relationships by penalizing redundant or uninformative latent variables, promoting a structured and interpretable embedding space. By integrating this loss into the JEPA framework, the proposed architecture not only improves computational efficiency but also aligns the latent space with meaningful data patterns, enabling better insights and downstream utility.

\label{sec:introduction}

\section{Background}
\subsection{Joint Embedding Predictive Architecture (JEPA)}
JEPA is a self-supervised learning framework designed to learn data representations. JEPA learns a shared embedding space by predicting the embedding of target patches (masked regions of an image) using the embeddings of context patches (visible regions), encouraging the model to infer relationships and semantic structures within the input data. This distinguishes JEPA from traditional reconstruction-based methods, as it directly learns from the latent space rather than raw pixel-level details \citep{assran2023selfsupervisedlearningimagesjointembedding}.
\subsection{Output Interpretable Variational Autoencoders (oi-VAE)}
oi-VAE \citep{pmlr-v80-ainsworth18a} is a framework designed to achieve interpretable and disentangled latent representations for nonlinear group factor analysis. It extends traditional variational autoencoders (VAEs) by grouping associating latent variables with specific components of the input data. This is achieved by decomposing the data into distinct outputs, each corresponding to a separate group of latent variables. The oi-VAE assumes that the observed data can be factorized into components, where each component is generated from a distinct latent variable. These latent variables are disentangled into group-specific and shared (global) factors, enabling the model to capture both unique and common structures in the data. By linking latent variables directly to data components, the oi-VAE achieves interpretability, allowing each latent dimension to correspond to a specific aspect of the data.
\subsection{Sparse Autoencoders}
Sparse Autoencoders are a type of neural network designed to learn efficient and meaningful representations of data by introducing a sparsity constraint in the hidden layer. This constraint encourages the network to activate only a small number of neurons in response to an input, resulting in compact and interpretable representations. Sparse Autoencoders aim to learn a representation where only a small subset of hidden units is active for any given input. This sparsity is typically enforced through a regularization term in the loss function, such as the KL divergence between the average activation of the hidden units and a desired sparsity target, or via $L_1$-regularization on the hidden activations \citep{le2012buildinghighlevelfeaturesusing}.
\label{sec:background}
\section{Sparse grouping of latent variables reduces redundant multiinformation}
In this section, we build up to a theorem which states that the grouping mechanism of SparseJEPA and oi-VAE enhances learned representations by reducing multiinformation among the data.\\\\
We follow \citet{8023479} and \citet{Studený1998} by defining multiiformation as:
\[
I(X_1; \ldots; X_n) = D_{\mathrm{KL}}\bigl(p(x_1, \ldots, x_n) \,\big\|\, \prod_{i=1}^n p(x_i)\bigr).
\]

where we let $X_1, \ldots, X_n$ be random variables with joint distribution $p(x_1, \ldots, x_n)$ and $D_{KL}(\cdot\,\big\|\ \cdot)$ is KL divergence.\\

\begin{lemma}[Grouping Reduces Multiinformation]
Let $X_1, \ldots, X_n$ be random variables with joint distribution $p(x_1, \ldots, x_n)$. Suppose the variables are partitioned into $m$ nonempty subsets $S_1, \ldots, S_m$. For each $j = 1, \ldots, m$, define a latent variable $G_j$ that depends only on the subset $\{X_i : i \in S_j\}$; that is,
\[
G_j = f_j\bigl((X_i : i \in S_j)\bigr),
\]
where $f_j$ is some deterministic function. Then, the multiinformation among the latent variables satisfies:
\[
I(G_1; \ldots; G_m) \leq I(X_1; \ldots; X_n).
\]
Furthermore, if there exist nontrivial inter-group dependencies in the original variables $X_1, \ldots, X_n$, the inequality is strict.
\end{lemma}

\begin{proof}
The mapping from $(X_1, \ldots, X_n)$ to $(G_1, \ldots, G_m)$ is deterministic because each $G_j$ depends only on a subset $S_j$ of the variables. By the data processing inequality for KL divergence, we have:
\[
I(G_1; \ldots; G_m) \leq I(X_1; \ldots; X_n).
\]
Moreover, the inequality is strict whenever the grouping process discards inter-group dependencies.
\end{proof}
\begin{theorem}
(Benefits of Sparse Grouping) Let $Z = (Z_1, \ldots, Z_k)$ be latent variables and $X = (X_1, \ldots, X_n)$ be observed variables. Suppose that $Z$ induces a structured dependency in $X$: each $Z_i$ predominantly influences a particular subset of $X$, and together $(Z_1, \ldots, Z_k)$ explain the joint distribution of $X$.

Now partition $\{Z_1, \ldots, Z_k\}$ into nonempty, disjoint subsets $S_1, \ldots, S_m$, and define the grouped variables:
\[
G_j = (Z_i : i \in S_j), \quad j = 1, \ldots, m.
\]
If the partition reflects the latent structure, then:
\begin{enumerate}
    \item Grouping reduces multiinformation:
    \begin{equation}
    I(G_1; \cdots; G_m) < I(X_1; \cdots; X_n),
    \end{equation}
    provided there exist nontrivial inter-group dependencies in $X_1, \ldots, X_n$.

    \item Grouping increases mutual information with $Z$:
    \begin{equation}
    I(Z; G_1, \ldots, G_m) \geq I(Z; X_1, \ldots, X_n).
    \end{equation}
\end{enumerate}
\end{theorem}

\begin{proof} We can proof each result as follows
\begin{enumerate}
    \item By Lemma 1, grouping reduces multiinformation:
    \[
    I(G_1; \cdots; G_m) \leq I(X_1; \cdots; X_n).
    \]
    Since grouping discards inter-group dependencies, the inequality is strict whenever such dependencies exist.

    \item Because $Z$ explains the joint structure of $X$, the grouped representation $G = (G_1, \ldots, G_m)$ captures the influence of $Z$ more effectively than the full set of variables $X$. Thus:
    \[
    I(Z; G_1, \ldots, G_m) \geq I(Z; X_1, \ldots, X_n).
    \]
\end{enumerate}
\end{proof}
In the context of representation learning, this theorem states there exists a transformation from data inputs to latent variables such that grouping latent variables into disjoint subsets can decrease the redundant multiinformation. This indicates that the latent space is able to store more relevant information.

\section{Method}
\begin{figure}[h]
    \centering
    \includegraphics[width=0.5\linewidth]{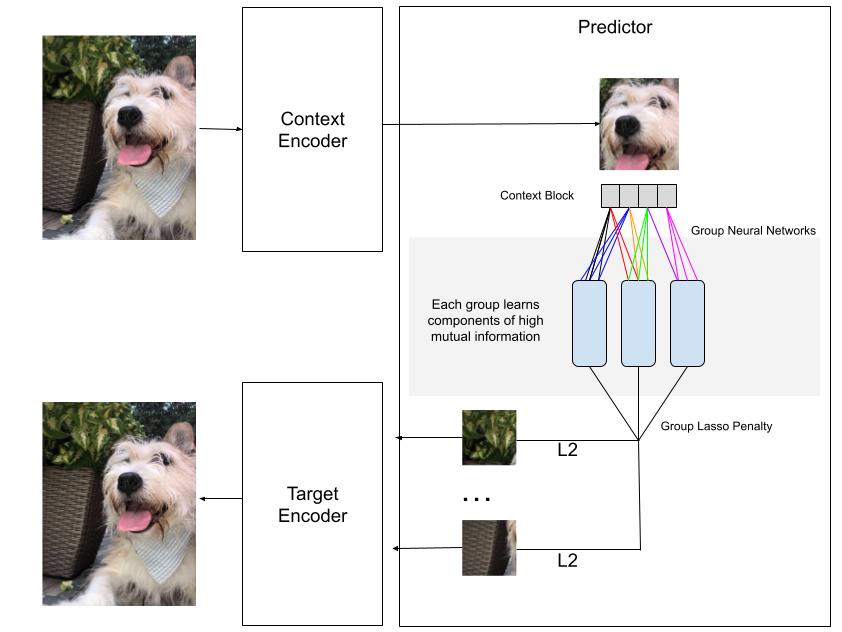}
    \caption{SparseJEPA Architecure}
    \label{fig:enter-label}
\end{figure}
This section describes the proposed SparseJEPA framework, which integrates sparse representation learning into JEPA. SparseJEPA employs a sparsity loss inspired by oi-VAE to improve interpretability downstream task performance.\par
\subsection{Overview of the SparseJEPA Architecture}
\begin{itemize}
    \item \textbf{Backbone Architecture}: A lightweight Vision Transformer (ViT) is used as the feature extractor, mapping input images to initial latent embeddings.
    \item \textbf{JEPA Prediction}: The framework learns a latent space embedding of the input by predicting target blocks given a context block in representation space, as described in the JEPA paper \citep{assran2023selfsupervisedlearningimagesjointembedding}.
    \item \textbf{Sparsity Module}: A sparsity-inducing penalty is applied to the latent space, encouraging embeddings to reflect meaningful groupings of shared semantic relationships.
\end{itemize}
\subsection*{Vision Transformer Description}
We fine-tune the Tiny Vision Transformer (ViT) implementation provided by the JEPA GitHub repository as the backbone for both the context encoder and the target encoder. The Tiny ViT operates by splitting the input image into non-overlapping patches, which are linearly embedded and augmented with positional encodings before being processed by a series of transformer layers. Each transformer layer consists of multi-head self-attention mechanisms and feedforward networks, enabling the model to capture both local and global features across the image. The resulting patch-level embeddings serve as a compact, high-dimensional representation of the image's semantic information. This lightweight version of the ViT is computationally efficient while retaining the ability to extract rich, abstract features, making it ideal for the JEPA framework where multiple encoders are used.

\subsection*{JEPA Description}
The JEPA leverages the Tiny ViT to learn predictive relationships between different regions of an image in representation space. The JEPA model divides the input image into context and target blocks, with the context block representing visible regions of the image and the target blocks representing unseen regions whose representations must be predicted. The embeddings generated by the Tiny ViT are used to encode both the context and target blocks. The predictor module then takes the context block embeddings, along with positional embeddings and mask tokens, to predict the high-level representations of the target blocks. 

To train the model, the loss function optimizes the difference between the predicted and actual target block embeddings. Specifically, the loss is defined as the average L2 distance between the predicted patch-level representations and their corresponding ground truth representations for each target block. For a target block $i$, represented by patches indexed by $j$ in block mask $B_i$, the loss is calculated as:

\[
\mathcal{L} = \frac{1}{M} \sum_{i=1}^{M} \sum_{j \in B_i} \lVert \hat{\mathbf{s}}_{y_j} - \mathbf{s}_{y_j} \rVert^2_2,
\]
where:
\begin{itemize}
    \item \( M \) is the number of target blocks.
    \item \( \hat{\mathbf{s}}_{y_j} \) is the predicted embedding for patch \( j \) in target block \( i \).
    \item \( \mathbf{s}_{y_j} \) is the ground truth embedding for the same patch.
\end{itemize}

This loss encourages the predictor to accurately reconstruct the semantic structure of unseen regions from the context, optimizing the model's ability to capture meaningful relationships within the representation space. By focusing on representation-level predictions rather than pixel-level reconstructions, JEPA achieves efficient and scalable learning of high-level semantic features.

\subsection{Sparsity Module}
\begin{figure}[h]
    \centering
    \includegraphics[width=0.5\linewidth]{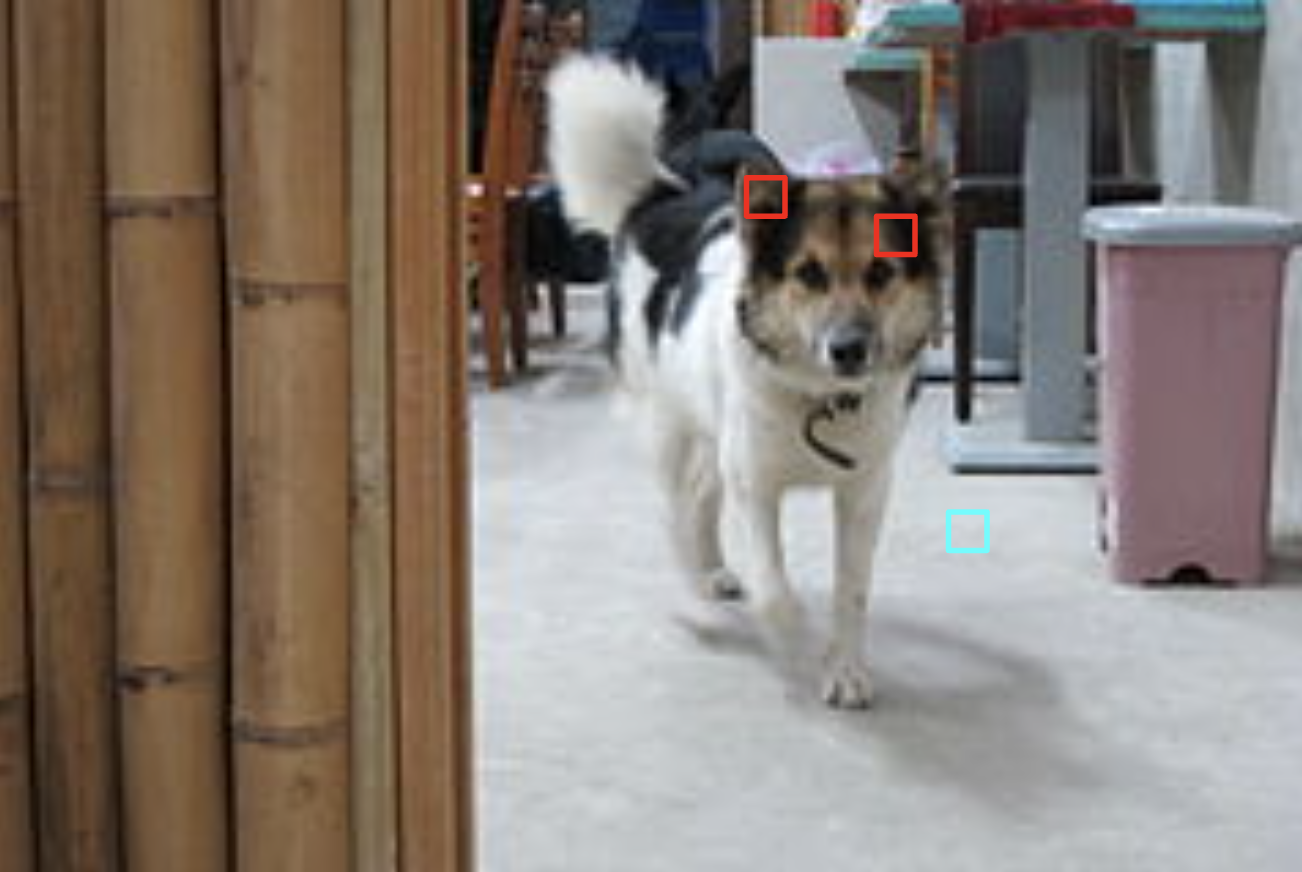}
    \caption{Image of a dog, where the two red patches contain more mutual information with eachother than the blue patch.}
    \label{fig:enter-label}
\end{figure}
We further refine the loss function described above by considering the intrinsic mutual information between pixels in an image. Based on the example in Figure 2, it may be advantageous to learn representations using the JEPA method, while maintaining that latent space dimensions of patches with significant mutual information are grouped together. This outcome is achieved by integrating a sparsity penalty in the prediction loss function, where latent variables are penalized for having heavily weighted in multiple groups. Formally, this new loss function can be described as:

\[
\mathcal{L} = \frac{1}{M} \sum_{i=1}^{M} \sum_{j \in B_i} \lVert \hat{\mathbf{s}}_{y_j} - \mathbf{s}_{y_j} \rVert^2_2 + \beta \, \mathcal{L}_{\text{KL}} + \lambda \sum_{g=1}^G \sum_{j=1}^K \| W^{(g)}_{\cdot, j} \|_2
\]

Where:
\begin{itemize}
    \item \( \lambda \): Regularization coefficient controlling the strength of the sparsity penalty.
    \item \( W^{(g)}_{\cdot, j} \): Latent-to-group matrix, where \( g \) indexes the groups and \( j \) indexes the latent variables.
    \item \( \beta \): Weight assigned to the KL divergence term.
\end{itemize}

This new formulation encourages sparsity by penalizing latent variables for influencing multiple groups, aiding in interpretability and disentanglement of group-specific interactions.\\\\
Based on our results in Section 3, we believe the improved model performance comes from the grouping mechanism causing latent dimensions to be shared among image patches with shared mutual information, capturing inherent meaning in representations. We believe this is caused by minimizing redundant information in the representation space, in which integrating sparsity has been shown to achieve \citep{FANG2018237}.

\label{sec:method}

\section{Results}
SparseJEPA was pre-trained using the CIFAR-100 dataset \citep{krizhevsky2009learning}, rather than ImageNet as proposed in the original JEPA paper due to computational constraints, however the updated architecture still shows a substantial improvement across benchmarks. To evaluate the performance of SparseJEPA, we trained the model on the CIFAR-100, Place205 \citep{NIPS2014_3fe94a00}, iNaturalist 2018 \citep{vanhorn2018inaturalistspeciesclassificationdetection}, and CLEVR \citep{johnson2016clevrdiagnosticdatasetcompositional}, datasets for the downstream task of top-1 image classification via simple linear probing. The reported results are based on our implementation of JEPA and SparseJEPA.

\par
\begin{table}[h!]
    \centering
    \caption{Top-1 classification accuracy for the proposed SparseJEPA model compared to baseline JEPA.}
    \begin{tabular}{|l|c|c|c|c|}
        \hline
        \textbf{Model} & \textbf{CiFAR100} & \textbf{Place205} & \textbf{CLEVR/COUNT} & \textbf{iNat18} \\
        \hline
        JEPA & 40.01 & 21.24 & 59.13 & 16.52 \\
        SparseJEPA & \textbf{45.4} & \textbf{23.36} & \textbf{62.33} & \textbf{19.63} \\
        \hline
    \end{tabular}
    \label{tab:classification_accuracy}
\end{table}
\label{sec:results}
\section{Conclusion}
In this paper, we introduced SparseJEPA, a novel extension to the JEPA framework, incorporating sparse representation learning to enhance interpretability and efficiency in latent space embeddings. SparseJEPA leverages a sparsity-inducing penalty inspired by the oi-VAE framework, which encourages the alignment of latent space dimensions with meaningful groupings of shared semantic relationships. By structuring the latent space in this way, SparseJEPA not only improves model interpretability but also achieves superior performance in downstream tasks, as demonstrated on the CIFAR-100 benchmark. This work highlights the potential of sparse representation learning to address the interpretability challenges inherent in dense embedding models, paving the way for more transparent and efficient self-supervised learning frameworks.
\label{sec:conclusion}
\section{Limitations}
While SparseJEPA demonstrates promising improvements in interpretability and performance, this work has two main limitations. First, we did not scale the model size due to computational constraints, which limited our ability to explore the full potential of SparseJEPA when applied to larger architectures. Scaling up could potentially lead to further performance improvements and more insights into the latent space structure, but it requires significant compute resources.
\label{sec:limitations}

\hfill

\break
\bibliography{bib}

\end{document}